\documentclass[anon]{colt2020} 

\title[Learn to Expect the Unexpected]{Learn to Expect the Unexpected:\\Probably Approximately Correct Domain Generalization}

\usepackage{times}
\coltauthor{%
 \Name{Vikas K. Garg} \Email{vgarg@csail.mit.edu}\\ 
 \addr MIT
 \AND
 \Name{Adam Kalai} \Email{adam.kalai@microsoft.com}\\ 
 \addr Microsoft Research
 \AND
 \Name{Katrina Ligett} \Email{katrina@cs.huji.ac.il}\\ 
 \addr Hebrew University
 \AND
 \Name{Zhiwei Steven Wu} 
 \Email{zsw@umn.edu}\\
 \addr University of Minnesota
}

\usepackage[utf8]{inputenc}
\usepackage{algorithm}
\usepackage[noend]{algpseudocode}
\usepackage{amsmath, amssymb, amsfonts}
\usepackage{mathtools}
\usepackage{color}
\usepackage{xcolor}
\usepackage{multicol}
\usepackage{bbm}
\usepackage{hyperref}
\usepackage{mleftright}
\usepackage{xspace}
\newtheorem{observation}[theorem]{Observation}

\DeclareMathOperator{\poly}{poly}

\newcommand\eps{\epsilon}
\newcommand\cA{\mathcal{A}}

\newcommand\cC{\mathcal{C}}
\newcommand\cD{\mathcal{D}}

\newcommand\cH{\mathcal{H}}

\newcommand\cL{\mathcal{L}}

\newcommand\cP{\mathcal{P}}

\newcommand\cT{\mathcal{T}}

\newcommand\cX{\mathcal{X}}
\newcommand\cY{\mathcal{Y}}
\newcommand\cZ{\mathcal{Z}}
\newcommand{\MDM}{\mathrm{MDM}}

\newcommand\reals{\mathbb{R}}
\newcommand\err{\mathrm{err}}

\newcommand\noiseless{\mathcal{P}_{shh}}
\newcommand\thresh{\mathrm{thresh}}

\newcommand\iid{iid}
\newcommand\train{\text{Tr}}
\newcommand\val{\text{Val}}

\newcommand{\E}{\mathop{\mathbb{E}}}
\DeclareMathOperator*{\argmin}{\mathrm{argmin}}

\begin{document}

\maketitle

\begin{abstract}
    Domain generalization is the problem of machine learning when the training data and the test data come from different data domains. We present a simple theoretical model of learning to generalize across domains in which there is a meta-distribution over data distributions, and those data distributions may even have different supports. 
    In our model, the training data given to a learning algorithm consists of multiple datasets each from a single domain drawn in turn from the meta-distribution. We study this model in three different problem settings---a multi-domain Massart noise setting, a decision tree multi-dataset setting, and a feature selection setting, and find that computationally efficient, polynomial-sample domain generalization is possible in each. Experiments demonstrate that our feature selection algorithm indeed ignores spurious correlations and improves generalization.
\end{abstract}

\section{Introduction}
Machine learning algorithms often fail to generalize in certain ways that come naturally to humans. For example, many people learn to drive in California, and after driving there for many years are able to drive in U.S.~states they have never before visited, despite variations in roads and climate. However, even a simple road-sign machine learning classifier would likely have decreased accuracy when tested on out-of-state road signs.

More generally, a common problem in real-world machine learning is that the training data do not match the test data. One well-studied instance of this issue is situations where the training and test data are drawn from different distributions over the same data domain. We are interested in a somewhat different problem---situations where \emph{the training data and the test data come from different (though potentially overlapping) domains}. A change in the data domain could occur because the underlying data distribution is changing over time, but it could also occur because an algorithm trained on data from a particular geographical location or context is later expected to perform in a different location or context.

While this problem of domain generalization has been studied empirically, our main contribution is a simple model of domain generalization in which theoretical results can be obtained. One challenge in formalizing the model is that arbitrary domain generalization is clearly impossible---an algorithm should not be expected to recognize a yield sign if it has never seen one (nor anything like it) before. We present a simple theoretical model of learning to generalize across domains in which there is a meta-distribution over data distributions, and those data distributions may have different domains (in the mathematical sense). In our model, the training data given to a learning algorithm consists of multiple datasets with each dataset drawn conditional on a single domain. The learning algorithm is expected to perform well on future domains drawn from the same distribution.

For example, there might be a meta-distribution over US states, and for each state there might be a distribution over examples, say features based on image and location latitude/longitude, taken in that state. The algorithm would be trained on multiple datasets---perhaps a Florida image dataset, a Wyoming image dataset, and an Alabama image dataset---and then would be expected to perform well not just on new images from Florida, Wyoming, and Alabama, but also on images from never-before-seen states. It may be, for example, that if each intersection had numerous visits, location features for predicting which type of sign is where, because signs rarely move. However, they may be seen not to generalize well across datasets.  

We then investigate this model in three quite distinct settings, and demonstrate that we can leverage the multi-domain structure in the problem to derive computationally efficient and conceptually simple algorithms. Our first result focuses on a multi-domain variant of the Massart noise model \citep{massart2006risk}, where there is a common target concept $c\in \cC$ across different domains but each domain has a different noise rate in the labels. We provide a general reduction from computationally efficient learning in this model to PAC learning under random classification noise \citep{AL87}. Our result can potentially provide new directions in resolving open questions in the standard Massart noise model where each individual example has its own label noise rate \citep{NIPS2019_8722}. See Section \ref{massart} for a discussion.

In our second result, we turn to another notoriously difficult computational problem---PAC learning decision trees. We make the assumption that there is a target decision tree that labels the examples across all domains, but examples in each domain all belong to a single leaf in this tree. Under this assumption, we provide an efficient algorithm with runtime $O(n+ s)$, where $n$ denotes the dimension of the data and $s$ denotes the number of nodes in the target tree. (Without any assumption, the fastest known algorithm runs in time $n^{O(\log s)}$).

Finally, our third result provides a simple algorithm for selecting features that are predictive across multiple domains. Our algorithm augments a black-box PAC learner with an additional correlation-based selection based on data across different domains. To empirically demonstrate its effectiveness, we also evaluate our algorithm on the ``Universities" dataset of webpages, for which the learning goal is to predict the category of each example (e.g., faculty, student, course, etc.). We show that our approach provides stronger cross-domain generalization than the standard baseline. As hypothesized, we find that features that are highly predictive in one university but not in another are in fact \textit{idiosyncratic}, removing them improves prediction on data from further universities not in the training set.   

We observe that our model of domain generalization enables two distinct advantages over the traditional PAC learning model. Most obviously, PAC-learned models do not come with any guarantee of performance on data points drawn from unobserved domains. Furthermore, the additional structure of training on multiple datasets enables in-sample guarantees that are not achievable in the PAC model.

\section{Related Work}

A rich literature sometimes known as domain adaptation~(e.g., \cite{DM06,BMP06,BBCP07,BCK+08,MMR09-1,MMR09-2,BBC+10,GL15,THSD17,MCM18,VMSM18}) considers settings where the learner has access not only to labeled training data, but also to unlabeled data from the test domain. This is a quite different setting from ours; our learner is given no access to data from the test domain, either labeled or unlabeled.

There is also a rich literature~(e.g., \cite{LZ08,LZX+08,CMW08,MMR09-3,GSB18}) that does not always rely on unlabeled data from the test distribution, but rather leverages information about similarity between domains to produce labels for new points. \cite{ZZY12}, relatedly, study the distance between domains in order to draw conclusions about generalization.

Adversarial approaches have recently gained attention \citep[e.g.][]{ZZW+18}, and in particular, \cite{VNS+18}, like us, generalize to unseen domains, but they attack the problem of domain generalization by augmenting the training data with fictitious, ``hard'' points.
There are also many other empirical approaches to the problem of domain generalization \citep[e.g.,][]{MBS13,khosla2012undoing,ghifary2015domain,li2017deeper,finn2017model,li2018learning,mancini2018best,BSC18,wang2019learning,carlucci2019domain,dou2019domain,li2019episodic}.

There are of course many other related fields of study, including covariate shift (wherein the source and target data generally have different distributions of unlabeled points but the same labeling rule), concept drift and model decay (wherein the distribution over unlabeled points generally remains static, but the labeling rule drifts over time), and multi-task learning (wherein the goal is generally to leverage access to multiple domains to improve performance on each of them, rather than generalizing to new domains).


\section{Definitions}

For mathematical notation, we let $[n]$ denote $\{1,2,\ldots,n\}$ and $1_Q$ denote the indicator function that is 1 if predicate $Q$ holds and 0 otherwise. For vector $x \in \reals^n$, let $x[k]$ denote the $k$th coordinate of $x$. Finally, let $\Delta(S)$ denote the set of probability distributions over set $S$. We now define our model of learning from independent datasets.

\subsection{Generalizing from multiple domains}

We consider a model classification with datasets from independent domains where training data $T=\langle T^1, \ldots, T^d\rangle\sim \rho_m^{\times d}$ consists of \textit{datasets} $T^i =\langle (x^i_1,y^i_1), \ldots, (x^i_m, y^i_m)\rangle$ each of $m$ examples. These $d$ datasets are chosen iid from \textit{dataset distribution} $\rho_m$ over $(\cX \times\cY)^m$. In particular, it is assumed that there is a distribution $\rho\in\Delta(\cX\times\cY\times\cZ)$ where $\cX$ is a set of examples, $\cY$ is a set of labels, and $\cZ$ is a set of \textit{domains}. Based on this $\rho_m$ selects $m$ labeled examples from a common latent domain as follows: $(x_1,y_1,z_1)$ is picked from $\rho$, and $(x_j,y_j)$ is picked from $\rho$ conditional on its domain being $z_j=z_1$ for $j\geq 2$. For simplicity, in this paper we will focus on classification with equal-sized datasets and latent domains but the model can be generalized to other models of learning, unequal dataset sizes, and observed domains. A \emph{domain-generalization learner} $L$ takes training data $T$ divided into of multiple datasets of examples as input and outputs classifier $L_T:\cX\rightarrow\cY$. $L$ is said to be computationally efficient if it runs in time polynomial in its input length. 

The \textit{error} of classifier $c: \cX\rightarrow \cY$ is denoted by $\err_\rho(c)=\Pr_{x,y,z\sim \rho}[c(x)\neq y]$ and $\rho$ may be omitted when clear from context. 
This can be thought of in two ways: $\err_\rho(c)$ is the expected error on $d'$ test datasets of $m'$ examples or  it is also the  average performance across domains, i.e., error rate on a random example (from a random domain) from $\rho$.

We first define a model of sample-efficient learning, for large $d$, with respect to a family $\cC$ of classifiers. Following the agnostic-learning definition of \cite{Kearns92towardefficient}, we also consider an \textit{assumption} $\rho \in \cP$ where $\cP$ is a set of distributions over $\cX \times \cY\times\cZ$.
\begin{definition}[Efficient Domain Generalization]\label{iid-learning}
Computationally-efficient domain-generalization learner $L$ is an \textit{efficient domain-generalization} learner for classifiers $\cC$ over assumption $\cP$ if there exist polynomials $q_d$ and $q_m$ such that, for all $\rho \in \cP$, all $\eps,\delta>0$, and all $d \geq q_d(1/\delta, 1/\eps), m \geq q_m(1/\delta, 1/\epsilon)$, 
$$\Pr\nolimits_{T \sim \rho_m^{\times d}}[\err_\rho\bigl(L_T\bigr) \leq \min_{c \in \cC} \err_\rho(c) + \epsilon] \geq 1-\delta.$$
\end{definition}
Standard models of learning can be fit into this model using iid and noiseless assumptions:
\begin{align*}
    \cP_{iid} &= \{ \rho \in \Delta(\cX\times\cY\times\cZ)~|~z \text{ is independent of } (x,y)\text{ for }x,y,z\sim \rho\}\\
    \noiseless(\cC) &= \{ \rho \in \Delta(\cX\times\cY\times\cZ) ~|~ \min\nolimits_{c\in \cC} \err_\rho(c)=0\}
\end{align*}
In particular, \textit{agnostic learning} can be defined as efficient domain-generalization learning subject to $\cP_{iid}$ while \textit{PAC learning} \citep{PAC} can be defined as efficient domain-generalization learning with $\cP_{PAC}=\cP_{iid} \cap \noiseless(\cC)$. 

It is not difficult to see that Definition \ref{iid-learning} is not substantially different from PAC and Agnostic learning, with a large number of datasets:
\begin{observation}
If $\cC$ is PAC learnable, then $\cC$ is efficiently domain-generalization learnable with noiseless assumption $\cP_{shh}(\cC)$. If $\cC$ is agnostically learnable, then $\cC$ is efficiently domain-generalization learnable without assumption, i.e.,  $\cP=\Delta(\cX \times\cY\times \cZ)$.
\end{observation}
\begin{proof}
Simply take a PAC (or agnostic) learning algorithm for $\cC$ and run it on the first example in each dataset. Since these first examples are in fact iid from $\rho$, the guarantees of PAC (or agnostic) learning apply to the error for future examples drawn from $\rho$.
\end{proof}
This is somewhat dissatisfying, as one might hope that error rates would decrease as the number of data points per domain increases. This motivates the following definition which consider the rate at which the error decreases separately in terms of the number of datasets $d$ and number of examples per dataset $n$.
\begin{definition}[Dataset-efficient learning]\label{dataset-efficient-learning}
Computationally-efficient learner $L$ is an \textit{dataset-efficient learner} for classifiers $\cC$ over assumption $\cP$ if there exists polynomials $q_d$ and $q_m$ such that, for all $\rho \in \cP$, all $\eps,\delta>0$, and all $d \geq q_d(1/\delta), m \geq q_m(1/\delta, 1/\epsilon)$, 
$$\Pr\nolimits_{T \sim \rho_m^{\times d}}[\err_\rho\bigl(L_T\bigr) \leq \min_{c \in \cC} \err_\rho(c) + \epsilon] \geq 1-\delta.$$
\end{definition}
This  definition requires fewer datasets than the previous definition, requiring a number of datasets that depends only on $1/\delta$ regardless of $\eps$.

In PAC and agnostic learning, many problems have a natural \textit{complexity parameter} $n$ where $\cX =\bigcup_{n \geq 1} \cX_n$, $\cC =\bigcup_{n \geq 1} \cC_n$, $\cZ = \bigcup_{n \geq 1} \cZ_n$, $\cP = \bigcup_{n \geq 1} \cP_n$, such as $\cX_n=\reals^n$. In those cases, we  allow the number of examples and datasets, $q_d, q_m$ in Definitions \ref{iid-learning} and \ref{dataset-efficient-learning}, to also grow polynomially with $n$. 
Also note that the set $\cP$ can capture a host of other assumptions, such as a margin between positive and negative examples. It is not difficult to see that the model we use is equivalent to a meta-distribution over domains $z$ paired with domain-specific distributions over labeled examples, where the domain-specific distributions would simply be the distribution $\rho$ conditioned on the given domain $z$.
Finally, while we assume that the chosen zones $z^i$ are not given to the learner---this is without loss of generality as the zones could be redundantly encoded in the examples $x$.

\section{Multi-Domain Massart Noise Model} \label{massart}

In the Massart noise model~\citep{massart2006risk}, each individual example $x$ has its own label noise rate that is, $\Pr[c(x)\neq y] = \eta(x) \leq \eta_b$, at most a given upper-bound $\eta_b<1/2$. Learning under
this model is computationally challenging and no efficient algorithms are known even for simple concept classes \citep{NIPS2019_8722}, despite the fact the statistical complexity of learning in this model is no worse than learning with noise rate $\eta$. We study a multi-domain variant of the Massart model, in which the learner receives examples with noisy labels from multiple domains such that each domain has its own fixed noise rate. We demonstrate that by leveraging the cross-domain structure of the problem we can obtain a broad class of computationally efficient algorithms. In particular, we provide a reduction from a multi-domain variant of the Massart noise model to PAC learning under random classification noise~\citep{AL87}. Let us first state the model formally as an assumption over the distributions $\Delta(\cX\times \cY \times \cZ)$.

\paragraph{Assumption $\cP_{\MDM}$.} There exists an unknown classifier $c\in \cC$ and an unknown noise rate function $\eta \colon \cZ \rightarrow \mathbb{R}$ such that the distribution $\rho$ over $\cX \times \cY \times \cZ$ satisfies $\Pr_\rho[ y \neq c(x) \mid z] = \eta(z) \leq \eta_b < 1/2$. We assume quantity $\eta_b$ is known to the learner.

Note that the minimal error rate $\Pr_\rho[y \neq c(x)] \in [0,\eta]$, achieved by the ``true'' classifier $c$, can be much smaller than $\eta$. Our multi-domain variant is a generalization in that the marginal distribution over labeled examples, ignoring zones, fits the Massart noise model. We will leverage the zone structure to provide a reduction from the learning problem in this model to PAC learning under classification noise, defined below.

\paragraph{PAC learning under {classification noise (CN)}~\citep{AL87}} 
Let $\rho_\cX$ be a distribution over $\cX$. For any
\emph{noise rate} $0 \leq \eta < 1/2$, the example oracle
$\text{EX}_{\mathrm{CN}}^\eta(c, \rho_\cX)$  on each call returns an example $(x, y)$ by first drawing an example $x$ from $\rho_\cX$ and then drawing a random noisy label $y$ such that $\Pr[y \neq c(x)] = \eta <\eta_b$, where $\eta_b$ is an known upper bound.
 The concept class $\cC$ is~\emph{CN learnable} if there exists a learning algorithm $\cL$ and a polynomial $f$ such that for any distribution $\rho_\cX$ over $\cX$, any noise rate $0\leq \eta < 1/2$, and for any $0 < \eps \leq 1$ and $0 < \delta \leq 1$, the following holds: $\cL$ will run in time bounded by  $f(1 / (1 - 2 \eta_b), 1/\eps, 1/\delta)$ and output a hypothesis $h$ that with probability at least $1 - \delta$  satisfies $\Pr_{x\sim \rho_\cX}[h(x) \neq c(x)] \leq \eps$.

\begin{theorem}
Let  $\cC$ be a concept class that is CN learnable. Then there exists an efficient domain generalization learner for $\cC$ under the multi-domain Massart assumption $\cP_\MDM$.
\end{theorem}
The basic idea behind the proof is to ``denoise'' data from each dataset by training a classifier within each dataset and then using that classifier to label another held-out example from that zone. If that classifier had high accuracy, then with high probability the predicted labels will not be correct. A noiseless classification algorithm can then be applied to the denoised data.  
\begin{proof}
Let $\cL$ be a CN learner for $\cC$ with runtime polynomial $f$. To leverage this learner to learn under the multi-dataset Massart model, we will aim to create an example oracle $\text{EX}_{\mathrm{CN}}^\eta$. Let $c \in \cC$ be the target concept, and let $\eps, \delta \in (0,1)$ be the target accuracy parameters. We will first draw a collection of  $d = f(1, 1/\eps, 2/\delta)$ datasets $T = \langle T^1, \ldots T^d\rangle$ from $\rho_m^{\times d}$, where $m > f(1/(1-2\eta_b), \delta/(4d), \delta/(4d))$. We will run the CN learner with a random subset of $T_i$ of size $(m-1)$ as input and obtain an hypothesis $h_i$ such that with probability $1 - \delta/(4d)$,
\begin{equation}
\Pr_{\rho_i}\left[ h_i(x) \neq c(x) \right] \leq \delta/(4d) ,  \label{pacman}
\end{equation}
where $\rho_i$ denotes the conditional distribution over $\cX$ conditioned on the zone being $z_i$. By a union bound, we know that except with probability $\delta/4$, \eqref{pacman} holds for all datasets $i$. We will condition on this level of accuracy (event $E_1$). Let $(x^i, y^i)$ denote an example in $T_i$ that was not used for learning $h_i$. This provides another dataset $\hat T = \langle (x^1, h_i(x^1)) , \ldots , (x^d, h_i(x^d)) \rangle$. Note that the $x^i$'s i.i.d. draws from the $\rho_\cX$, the marginal distribution of $\rho$ over $\cX$. Furthermore, by the accuracy guarantee of each $h_i$, $\Pr[h_i(x^i) \neq c(x^i)] \leq \delta/(4d)$. By a union bound, we know that except with probability $\delta/4$, $h_i(x^i) = c(x^i)$ for all $i\in[d]$. We will condition on this event of correct labeling (event $E_2$). This means the examples in  $\hat T$ can simulate random draws from $\mathrm{EX}_{\mathrm{CN}}^0(c, \rho_\cX)$. Finally, we will run $\cL$ over the set $\hat T$, and by our choice of $d$, $\cL$ will output a hypothesis $h$ such that $\Pr_\rho[h(x)\neq c(x)] \leq \eps$ with probability at least $1-\delta/2$ (event $E_3$). Finally, our learning guarantee follows by combining the failure probability of the three events $E_1, E_2, E_3$ with a union bound. 
\end{proof}

\paragraph{Open problem in the (multi-domain) Massart model.} An open question in the multi-domain Massart noise model is whether there exists an efficient algorithm that only relies on a constant number of examples from each domain. If we can decrease the number of examples in each domain down to 1, we recover the standard Massart noise model. Thus, we view this as an intermediate step towards an efficient algorithm for the standard Massart model \citep{NIPS2019_8722}.

\section{Decision Tree Multi-Dataset Model}

We next consider learning binary decision trees on $\cX=\{0,1\}^n$. Despite years of study, there is no known polynomial-time PAC learner for decision trees, with the fastest known algorithm learning binary decision trees of size $\leq s$ in time $n^{O(\log s)}$ \citep{hellerstein2007pac}.  Formally, a decision tree is a rooted binary tree where each internal node is annotated with an attribute $1 \leq i \leq n$, and the two child edges are annotated with 0 and 1 corresponding to the restrictions $x[i]=0$ and $x[i]=1$. Each leaf is annotated with a label $\{0,1\}$, and on $x$ the classifier computes the function that is the label of the leaf reached by following the path starting at the root of tree and following the corresponding restrictions. 

\paragraph{Assumption $\cP_{DT}(s, n)$.}
Let $\cT_{s, n}$ be the class of decision trees with at most $s$ leaves. The domains simply correspond to the leaves of the tree in which the (noiseless) example belongs. To make this assumption denoted $\cP_{DT}(s, n)$ formal, let the set of domains $\cZ$ is simply the set of all $3^n$ possible \textit{conjunctions} (each $x[j]$ can appear as positive, negative, or not at all) on $n$ variables. We identify each leaf $\ell$ in a tree with domain $z_\ell\equiv x[j_1]=v_1 \wedge x[j_2]=v_2\wedge \ldots \wedge x[j_k]=v_k$, where $k$ is the depth of the leaf, $j_1,j_2, \ldots, j_k\leq n$ are the annotations of the internal nodes on the path, and $v_k \in \{0,1\}$ correspond to the edges on the path to that leaf. Using this notation, the assumption $\cP_{DT}(s, n)$ is that there is a tree $T \in \cT_{s, n}$ for which, with probability 1 over $\rho$, every example $(x,y,z)$ satisfies $z=z_\ell$ for the leaf $\ell$ in $T$ which $x$ belongs to, i.e., conjunction $z_\ell$ holds, and $y=T(x)$, i.e., noiselessness $\cP_{DT}(s, n) \subset \cP_{shh}(\cT_{s,n})$. 

Recall that the chosen domains $z^i$ themselves are not observed, otherwise the learning problem would be trivial.
Also note that the natural algorithm that tries to learn a classifier for each dataset to distinguish those examples from examples in other datasets will not work because multiple datasets may represent the same leaf (zone). Instead, we leverage the fact that conjunctions can be learned from positive examples alone. 

In particular, we think of the decision tree simply as the union (OR) of the conjunctions corresponding to leaves labeled positively. It is known to be easy to PAC-learn conjunctions \textit{from positive examples alone} by outputting the \textit{largest consistent conjunction} \cite[][Section 1.3]{kearns1994introduction}: the hypothesis given by the conjunction of the subset of  possible terms $\{x[j]=b \mid j \in [n], b\in\{0,1\}\}$ that are consistent with every positively labeled example.\footnote{For example, for the two positive examples $(0,0,1)$ and $(0,1,1)$, the largest consistent conjunction is $x[1]=0 \wedge x[3]=1$.} It is largest in terms of the number of terms, but it is minimal in terms of the positive predictions it makes, and it never has any false positives. The following algorithm learns decision trees in the above multi-dataset decision tree model.
\begin{enumerate}
    \item Input: training data $T^1, T^2, \ldots, T^d$ .
    \item Let $\textsc{PositiveDomains} = \{i~|~y^i_1=1\}$.
    \item For each $i \in \textsc{PositiveDomains}$, find the largest consistent conjunction $c_i$ for $T^i$.
    \item Output the classifier $\hat{c}(x)=\begin{cases}1 & \text{if }c_i(x)=1\text{ for any } i \in \textsc{PositiveDomains}\\0 & \text{otherwise}.\end{cases}$
\end{enumerate}
\begin{theorem}\label{thm:decisionTrees}
Let $s, n \geq 1$ and $\cT_s$ be the family of binary decision trees of size at most $s$ on $\{0,1\}^n$. Then the above algorithm is an efficient domain-generalization learner for $\cP_{DT}(s, n)$ for complexity parameter $N=n+s$.
\end{theorem}
For decision trees, the complexity of the class depends on both the number of variables and the size of the tree, hence we use $N=n+s$ as a complexity measure.
\begin{proof}
For high-probability bounds, it suffices to guarantee \textit{expected} error rate at most $\eps\delta$, for $m,d\geq q(\frac{sn}{\eps\delta})$, for some polynomial $q$, by Markov's inequality.

First, it is not difficult to see that the algorithm will never have any false positives, i.e., it will never predict positively when the true label is negative. To see this, note that each positive prediction must arise because of at least one $c_i$. As mentioned above, \cite{kearns1994introduction} show that the largest consistent classifier with any set of (noiseless) positive data is conservative in that it never has any false positives. Hence the above algorithm will never have any false positives.

We bound the expected rate of false negatives (which is equal to the expected error rate) by summing over leaves and using linearity of expectation. False negatives in positive leaf $\ell$ can arise in two ways: (a) leaf $\ell$ was simply never chosen as a domain, and (b) leaf $\ell$ was chosen $z^i=\ell$ for some $i \leq d$, but there is a term $x[j]=k$ for some $j\leq n, k \in \{0,1\}$ which occurs in $c_i$ but not in $z_\ell$ in which case any positive example that satisfies $x[j]=k$ will be a false negative. Moreover, these are the only types of false negatives. Hence, the expected rate of false negatives coming from leaf $\ell$ with probability $p_\ell$ due to (a) is $p_\ell(1-p_\ell)^d$, the fraction of examples from leaf $\ell$ times the probability that domain $z_\ell$ was never chosen. The expected rate of false negatives due to (b) is at most $p_\ell 2n/(m+1)$, again the probability of leaf $\ell$ times $2n/(m+1)$. To see why, note that there are at most $2n$ terms $(x[j]=k)$ not in the true conjunction $z_\ell$ and, for each such term, the expected error contribution can be upper bounded by imagining picking $m+1$ examples at random, $m$ for training and 1 for test. The probability that among $m+1$ positive examples, that only example which would satisfy that term would be the one chosen for test is $1/(m+1)$. Hence the expected rate of false negatives and hence also the expected error rate is at most
\begin{equation}\label{eq:dt}
\sum_\ell p_\ell (1-p_\ell)^d  + p_\ell\frac{2n}{m+1} < \frac{s}{d} + \frac{2n}{m}.
\end{equation}
The inequality above holds for the left term because $r(1-r)^d\leq 1/d$ for $r\in [0,1]$ and for the right term because the $\sum p_\ell=1$ and for the left term by concavity of $\sum p_\ell(1-p_\ell)^d$ on the probability simplex. Note that the above error rate is bounded by $\eps\delta$ if we have $d \geq 4s/(\eps \delta)$ and $m \geq 4n/(\eps \delta)$, which completes our proof.
\end{proof}

\section{Feature Selection Using Domains}\label{sec:feature-selection}
Finally, we use access to training data from multiple domains to aid in performing feature selection.

In this section, we fix $\cX=\{0,1\}^n$.
For set $R \subseteq[d]$, let $x[R]=\langle x[k]\rangle_{k\in R}\in \{0,1\}^{|R|}$ denote the selected features $R$ of example $x\in\cX$.
Let $z^i$ denote the domain corresponding to training dataset $T^i$, for each $i \in [d]$. Define $\rho_k$ to be the correlation of $x[k]$ and $y$ over $\rho$ and let $\rho^i_k$ denote the usual (Pearson) correlation coefficient of feature $x[k]$ with $y$ conditioned on the example having domain $z=z^i$. Let $\hat{\rho}^i_k$ denote the empirical correlation of $x[k]$ and $y$ on $T^i$.

The following algorithm (\textbf{FUD}) performs feature selection using domains.
\begin{enumerate}
    \item Input: class $\cC$, parameters $\beta, \eps \geq 0$, training data $T$ consisting of $d$ splits of $m$ examples each.
    \item If the overall fraction of positive or negative examples is less than $\eps/2$ (massive class imbalance), stop and output the constant classifier $c(x)=0$ or $c(x)=1$, respectively.
    \item For each variable $i\in[n]$, compute empirical correlation $\hat\rho_k^i$ of $x[k]$ and $y$ over each dataset $i \in [d]$.
    \item Let $R=\left\{k~|~\min_i |\hat{\rho}^i_k|\geq \beta \right\}$.
    \item Find any $c \in \cC$ such that $c(x[R])=y$ for all $s,x,y \in T$, and output classifier $f(x)=c(x[R])$. If no such $c$ exists, output \textit{FAIL}.
\end{enumerate}

\paragraph{Assumption  $FS(\cC,\beta)$} 
For $\beta>0$ we define the Feature Selection assumption $FS(\cC,\beta)$ to require that there exists a robust set of features $R \subseteq [d]$ such that:
\begin{itemize}
    \item Noiselessness $\noiseless(\cC)$: For some $c \in \cC$, $\Pr_\rho[c(x[R])=y]=1$.
    \item Independence: $x[R]$ and $z$ are independent over $\rho$.
    \item Correlation: For all $k \in R$, $|\rho[k]| > 1.1\beta$
    \item Idiosyncrasy
: For all $k \not\in R$, $\Pr_{x,y,z\sim \rho}\bigl[~|\rho^z_k| < 0.9\beta\bigr] > 0.1$.
\end{itemize}
Note that the constants 1.1, 0.9 and 0.1 in the above assumption can be replaced by parameters (e.g., $1\pm \eps_1$ and $\eps_2$) and the dependence of $d$ and $m$ on these parameters in the following theorem would be inverse polynomial.
\begin{theorem}\label{thm:FUD}
For any $\cC$ of finite VC dimension $VC(\cC)$, with $\cX = \{0,1\}^n$, $\cY=\{0,1\}$ and any $\beta> 0$, FUD is a dataset-efficient learner under assumption $FS(\cC,\beta)$. In particular, for $d=O\left(\log \frac{n}{\delta}\right)$ and $m=O\left(\frac{VC(\cC)}{\eps}+ \frac{\log(n/\delta)}{\beta^4\eps^2}\right)$, 
$$\Pr_T[\err_\rho(\text{FUD}_T) \leq \eps] \geq 1-\delta,$$
for any $\eps,\delta\in (0,1/2)$.
\end{theorem}

\begin{proof}
Fix  $\rho \in FS(\cC, \beta)$. Note that by the noiseless and independent assumptions, the fraction of positives is the same in each domain, i.e., $\E[y|z]=\E[y]$. 
We first bound the failure probability of outputting the all 0 or all 1 classifier in the second step. However, if $\E[y]\geq \eps$, the probability that it outputs the all 0 classifier is at most $\delta/10$ by multiplicative Chernoff bounds  over $dm=\Omega(\frac{1}{\eps}\log \frac{1}{\delta})$ labeled examples. Similarly, if $\E[y]\leq 1-\eps$, the probability we output the all 1 classifier is at most $\delta/10$. Conversely, if $\E[y]<\eps/4$, then multiplicative Chernoff bounds also imply that with probability at least $1-\delta/10$, we will output the 0 classifier (and hence have error $<\eps$), and similarly if $\E[y]>1-\eps/4$.

Henceforth, let us assume $\E[y] \in [\eps/4,1-\eps/4]$. 

Next, note that the set $R$ described in the $FS$ assumption is uniquely determined for $\rho$. Call this set $R^*$. It suffices to show that  with probability at least $1-\delta/10$, $R=R^*$ for $R$ defined in the algorithm. This is because if $R=R^*$, by a standard VC bound of \cite{haussler1991equivalence}, since $x[R]$ is iid and the total number of examples observed is $dm=\Omega\left(\frac{VC(\cC)}{\eps}\log \frac{1}{\delta}\right)$, with probability at least $1-\delta/2$ the error is at most $\eps$ because learning of $(x[R], y)$ is standard PAC learning of $\cC$.

Using $\E[y] \in [\eps/4,1-\eps/4]$, Lemma \ref{lem:corr} below implies that $m=\Omega(\beta^{-4}\eps^{-2}\log(dn/\delta))$ examples suffice to estimate all $dn$ correlations accurately to within $0.1\beta$ with probability at least $1-\delta/10$. Assuming this happens, all $k\in R^*$ will necessarily also be in $R$. 

It remains to argue that with probability at least $1-\delta/10$, $R=R^*$. To see this, note that for each $k \not\in R^*$, the Idiosyncrasy assumption means that with probability at most $0.9^d \leq \delta/(10n)$ would there be no $k$ for which $|\rho_i^k|\leq 0.9\beta$. Hence, by a union bound, with probability at least $1-\delta/10$, there will be simultaneously for each $k \not\in R^*$ some dataset $i\in [d]$ such that $|\rho_i^k|\leq 0.9\beta$. Since we are assuming that all correlations are estimated correctly to within $0.1\beta$, it is straightforward to see that $R=R^*$.\end{proof}

We now bound the number of examples needed to estimate correlations.

\begin{lemma}\label{lem:corr}
For any jointly distributed binary random variables $(R,S) \in \{0,1\}^2$ with $\E[S] \in [v, 1-v]$, and for any $\eps,\delta >0$, the probability that the empirical correlation coefficient of $m\geq 2048 \eps^{-4} v^{-2} \log(8/\delta)$ iid samples differs by more than $\eps$ from the true correlation is at most $\delta$. 
\end{lemma}
The proof of this Lemma is given in Appendix \ref{ap:fssproof}.

\begin{table}
\caption{Data statistics \label{tab:Table1}}
\vskip 0.2cm
\begin{center}
 \begin{tabular}{|l| c| c| c|} 
 \hline
 {\bf Domain} & {\bf Pages} & {\bf Faculty proportion} & {\bf Bag density} ({\bf student pages}, {\bf faculty pages})
  \\ [0.5ex] 
 \hline\hline
 {Cornell} & 162 & 21\% & 23\% (22\%, 28\%) \\ 
 \hline
 {Texas} & 194 & 24\% & 23\% (23\%, 22\%) \\
 \hline
{Washington} & 157 & 20\% & 24\% (24\%, 20\%) \\
 \hline
 {Wisconsin} & 198 & 21\% & 23\% (21\%, 29\%) \\
 \hline\hline
 {Test} & 2,054 & 47\% & 21\% (22\%, 21\%) \\
 \hline
\end{tabular}
\end{center}
\end{table}

\section{Experiments}
\begin{figure}[!h]
    \centering
    \includegraphics[width=15cm]{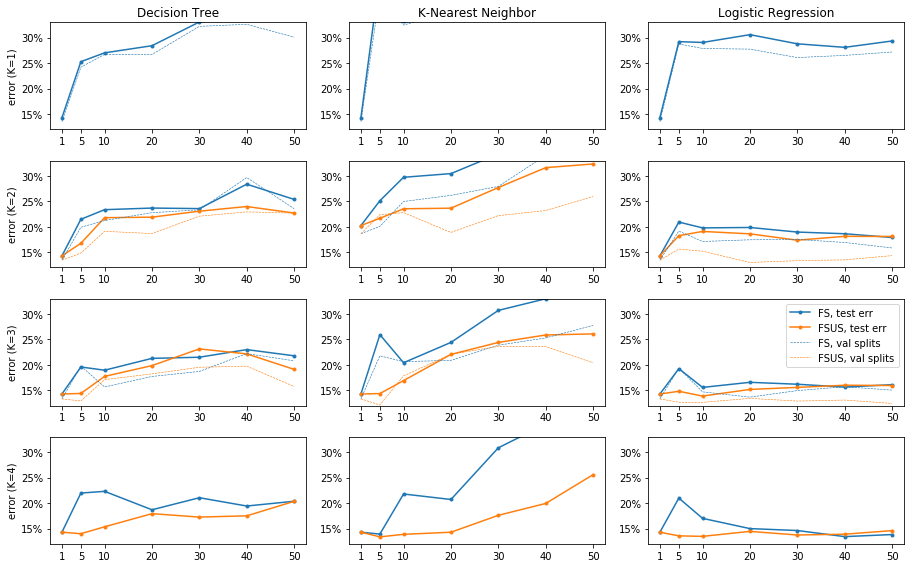}
    \caption{Balanced error rates on University data for varying number of selected features. FSUS is our algorithm, and the baseline is denoted by FS.}
    \label{fig:Figure1}
\end{figure}

\begin{figure}[!h]
    \centering
    \includegraphics[width=15cm]{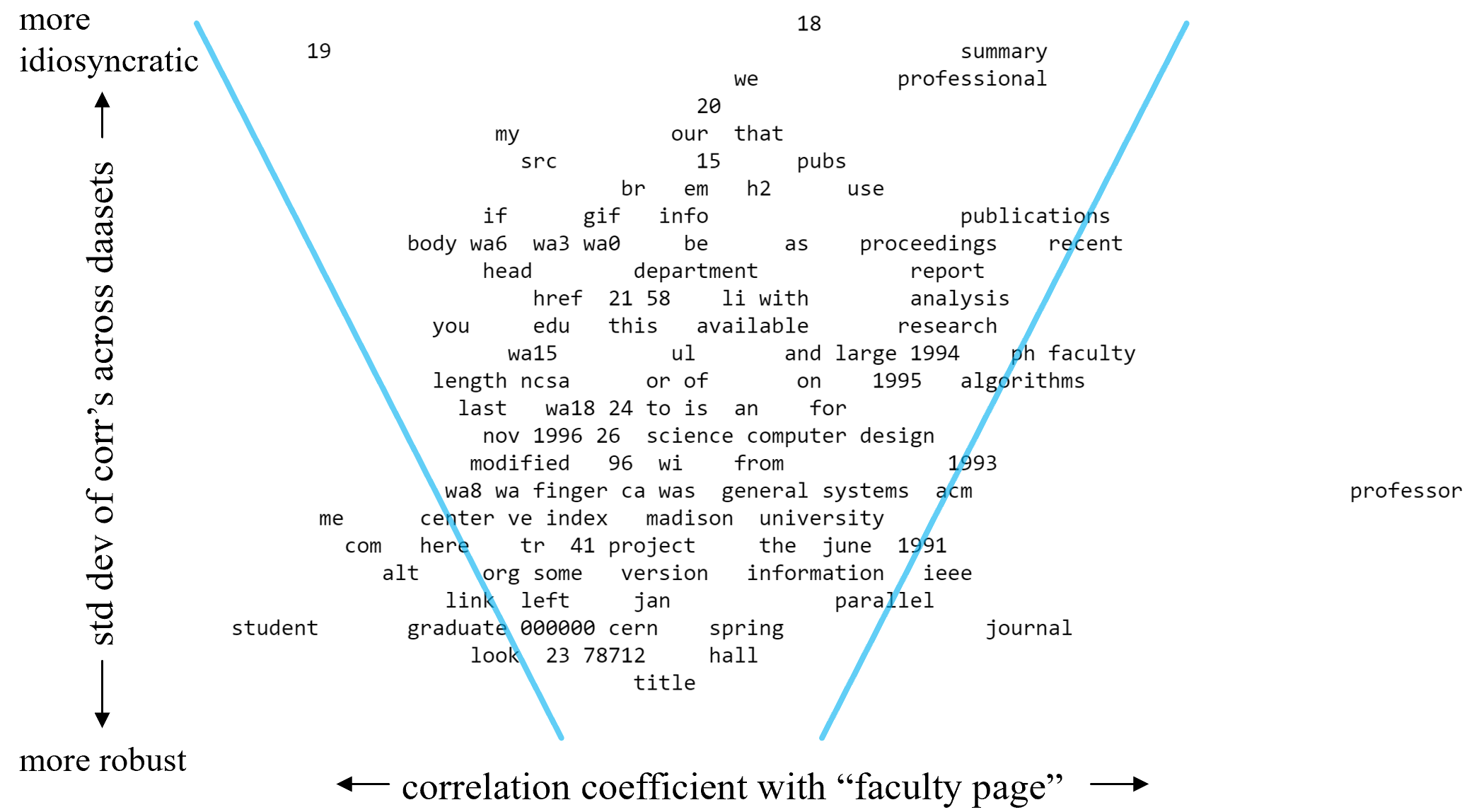}
    \caption{Correlations of words with faculty page ($x$-axis) vs the std.~dev.\ of correlations over universities. Words to right such as \textit{professor} and \textit{journal} correlate most strongly with faculty pages, while words to the left such as \textit{student} and \textit{19} correlate most with student pages. Words towards the bottom such as \textit{student} have robust correlations across universities while words towards the top such as \textit{19} are more idiosyncratic. (Interestingly, the token \textit{19} was found to consistently predict student because the web pages headers were included in the data and, coincidentally, the time of download of student web pages for some universities was 7pm.) The words selected are those outside the diagonal lines, where the slope of the line is determined by parameter $\alpha$, and the horizontal positions of the lines is determined by the number of words to be selected. }
    \label{fig:scatterplot}
\end{figure}

We conducted simple experiments to evaluate the quality of features selected by our methodology from Section~\ref{sec:feature-selection}. We experimented with the {\em Universities} data set,\footnote{\textbf{Available at}: http://www.cs.cmu.edu/afs/cs/project/theo-20/www/data/} a small dataset that is ideally suited for domain generalization. It conatins webpages from computer science departments of various universities, which can be identified by the url domain, e.g., \url{cornell.edu}. The data set is classified into categories such as faculty, student, course, etc.; we focused on the faculty and student classes for our experiments. Our training data pertains to 711 faculty and student webpages from four universities: Cornell, Texas, Washington, and Wisconsin.
 Our test data includes faculty and student pages from 100 universities. None of the four universities in our training set were represented in the test set. We represented each page 
as a bag-of-words, and preprocessed the data to remove all words that had less than 50 occurrences. As a result, we obtained a vocabulary of 547 unique words. Thus, we represented each page as a 547-dimensional binary vector:  each word that occurred at least once in the page had the corresponding coordinate set to 1. 

We summarize the statistics of our data in Table \ref{tab:Table1}. Note that we computed the bag density of a domain as the average of the mean vector pertaining to the binary vectors in the domain. The respective densities for student and faculty pages are also shown.   Note that the faculty proportion in test data (47\%) is about twice the proportion in any domain from the training data (where the fraction of faculty pages hovers around 20\%). Thus, investigating this data for domain generalization is a worthwhile exercise. \\

 We compare the performance of our algorithm with a standard feature baseline. Specifically, the baseline selects words whose Pearson correlation coefficient with the training labels (i.e., faculty or student) is high. We implemented a regularized version of our feature selection algorithm FUD that penalized those features that have large standard deviation (stdev) of the Pearson coefficient on the train domains. In other words, we computed scores $s_k = |\hat{\rho}_k| - \alpha~~ {\rm stdev}(\hat{\rho}_k^1, \ldots, \hat{\rho}_k^d)$, and selected the features $k$ that were found to have high $s_k$. We set the value of the regularization parameter $\alpha$ to 2.  We call our regularized algorithm FSUS. We trained several classifiers, namely, decision tree, K-nearest neighbor, and logisitic regression, on the features selected by each algorithm  (using default values of hyperparameters in the Python {\em sklearn} library). The performance of the algorithms was measured in terms of the standard  {\em balanced error rate}, i.e., the average of prediction error on each class.  Besides the performance on test data, we also show the mean validation error to estimate the generalization performance on domains in the training set.  Specifically, we first trained a separate classifier for each domain and measured its prediction error on the data from other domains in the training set, and then averaged these errors to compute the estimate of validation error, denoted by (K=1) in Figure \ref{fig:Figure1}. Likewise, for $K=2$, classifiers were trained on data from two domains at a time, and evaluated for performance on the other domains; similarly for $K \in \{3, 4\}$. As Figure \ref{fig:Figure1} illustrates, our algorithm generally outperformed the baseline method, for different numbers of selected features (horizontal axis) and for different $K$ across classifiers. Note that instead of fixing $\alpha$ beforehand, we could tune it based on the validation error. We found that performance of our algorithm deteriorated only slightly using the tuned $\alpha$. We omit the details for brevity.  These empirical findings substantiate our theoretical foundations, suggesting the benefits of domain generalization.    

Figure \ref{fig:scatterplot} shows a scatter-plot of the correlations of features, words in this instance, and robustness of this correlation across datasets. Interestingly, one of the most correlated features was the token \textit{19}, which was later discovered to be correlated in certain datasets simply because student webpages at certain universities were downloaded at 7pm, and the datafiles included header information which revealed the download times.  It is normally considered the job of a data scientist to decide to ignore features such as data collection time, but this illustrates how our algorithm identified this problem automatically using the idea of robustness across domains.

\section{Conclusions and Open Directions}

The goal of this paper is to suggest a simple theoretical model of domain generalization, and to demonstrate its power to obtain results that leverage access to multiple domains.

Even in settings where training data are not explicitly partitioned into domains, ideas from this work can potentially be helpful in developing algorithms that will be robust to unfamiliar data. One approach is to create splits of the training data based on clustering it or dividing it along settings of its variables, such that a domain expert believes that the resulting division into splits may be analogous to future changes in the data to be handled. (Some of the training data could even potentially be used to test out the usefulness of a candidate partition into splits.)

\bibliography{bibfile}

\appendix

\section{Proving bounds on correlation coefficients}\label{ap:fssproof}
This section includes the proof of Lemma \ref{lem:corr}.

\begin{proof}[Proof of Lemma \ref{lem:corr}]
Let $\rho$ and $\hat{\rho}$ be the correlation coefficient and empirical correlation of $R,S$ on a sample of size $m$. Let $q_{ij}=\Pr[R=i \wedge S=j]$ and $\hat{q}_{ij}$ be the corresponding realized empirical fractions over the $m$ samples. 

By Chernoff bounds, for any $i,j$, the probability that $|q_{ij}-\hat{q}_{ij}|> \tau$ is at most $2e^{-2m\tau^2}\leq \delta/4$ for $\tau=\eps^2 v/64$. Hence, with probability $\geq 1-\delta$, $|q_{ij}-\hat{q}_{ij}|\leq \tau$ for  and for all $i,j$. We now argue that if this happens, then $|\rho-\hat{\rho}|\leq \eps$. 

As shorthand, let $a=q_{00}, b=q_{01}, c=q_{10}, d=q_{11}$ and $\hat{a},\hat{b},\hat{c},\hat{d}$ be the analogous empirical quantities. It may be helpful for the reader to draw a 2x2 table of possible values of $R,S$ and associated probabilities.

Case 1: $c+d \leq \tau$. In this case we use $|\rho-\hat\rho|\leq |\rho|+|\hat\rho|$ and argue that both  $|\rho|,|\hat\rho|\leq 2\sqrt{\tau/v}\leq \eps/2$. To see this, the definition of correlation coefficient applied to binary random variables means that correlation can be written as
\begin{equation}\label{eq:corr}
\rho = \frac{a d-b c}{\sqrt{(a+b)(c+d)(a+c)(b+d)}},
\end{equation}
and similarly for $\hat\rho$. Since all quantities are non-negative, we can remove terms to get
$$-\sqrt{\frac{c}{a}}\leq \frac{-b c}{\sqrt{(b)(c)(a)(b)}}\leq \rho \leq \frac{a d}{\sqrt{(a)(d)(a)(b)}} = \sqrt{\frac{d}{b}},$$
and similarly for $\hat\rho$. In turn this implies that $|\rho|\leq \max \{\sqrt{c/a},\sqrt{d/b}\}$. Since $c+d \leq \tau$, we have that $c,d \leq \tau$ and since $\E[S]\in [v, 1-v]$ we have that $a+c,b+d\geq v$, in turn implying $a, b \geq v-\tau$. Hence, 
$$|\rho|\leq \max\left\{\sqrt{\frac{{c}}{{a}}}, \sqrt{{\frac{d}{b}}}\right\}\leq \sqrt\frac{\tau}{v-\tau}.$$
Similarly, for $\hat\rho$, we have
$$|\hat\rho| \leq \max\left\{\sqrt{\frac{\hat{c}}{\hat{a}}}, \sqrt{\frac{\hat{d}}{\hat{b}}}\right\}\leq \max\left\{\sqrt{\frac{c+\tau}{a-\tau}}, \sqrt{\frac{d+\tau}{b-\tau}}\right\} \leq \sqrt\frac{\tau+\tau}{v-\tau-\tau}\leq \sqrt\frac{2\tau}{v-2\tau}.$$
This upper bound is greater than the one we have for $|\rho|$. Hence,
$$|\rho-\hat\rho|\leq |\rho|+|\hat\rho| \leq 2 \sqrt\frac{2\tau}{v-2\tau}\leq2\sqrt\frac{2\tau}{v/2} = 4 \sqrt\frac\tau{v}\leq \eps.$$ 
In the above we have used the fact that $2\tau\leq v/2$.

Case 2: $c+d  \in [\tau, 1/2]$. We use the fact that, given that $|\hat{a}-a|, |\hat{b}-b|\leq \tau$, 
\begin{equation}\label{eq:foo}\frac{\hat{a}}{\hat{a}+\hat{b}} \leq \frac{a+\tau}{(a+\tau)+(b-\tau)} = \frac{a+\tau}{a+b},\end{equation}
because $\alpha/(\alpha+\beta)$ is increasing in $\alpha$ and decreasing in $\beta$. 
From (\ref{eq:corr}), one can see that
\begin{equation}\label{eq:baz} \rho = \sqrt{ \frac{a}{a+b}\cdot \frac{d}{c +d}\cdot \frac{a}{a+c}\cdot\frac{d}{b+d}}-\sqrt{ \frac{b}{a+b}\cdot \frac{c}{c +d}\cdot \frac{c}{a+c}\cdot\frac{b}{b+d}}.\end{equation}
The bounds on $\hat{a}/(\hat{a}+\hat{b})$ imply that
\begin{align}
\sqrt{ \frac{\hat{a}}{\hat{a}+\hat{b}}\cdot \frac{\hat{d}}{\hat{c} +\hat{d}}\cdot \frac{\hat{a}}{\hat{a}+\hat{c}}\cdot\frac{\hat{d}}{\hat{b}+\hat{d}}} &\leq \sqrt{ \frac{a+\tau}{a+b}\cdot \frac{d+\tau}{c +d}\cdot \frac{a+\tau}{a+c}\cdot\frac{d+\tau}{b+d}}\nonumber\\
&=\frac{(a+\tau)(d+\tau)}{\sqrt{(a+b)(c+d)(a+c)(b+d)}}\nonumber\\
&\leq \frac{a d + 2\tau}{\sqrt{(a+b)(c+d)(a+c)(b+d)}}.\label{eq:bar}
\end{align}
In the last step above, we used the fact that $\tau(a+d)\leq \tau$ since $a+d\leq 1$. Similarly to (\ref{eq:foo}), we have
$$\frac{\hat{a}}{\hat{a}+\hat{b}} \geq \frac{\max\{0,a-\tau\}}{a+b}.$$
Combining with a similar lower bound to (\ref{eq:bar}) gives
$$\left|\frac{\hat{a} \hat{d}}{\sqrt{(\hat{a}+\hat{b})(\hat{c}+\hat{d})(\hat{a}+\hat{c})(\hat{b}+\hat{d})}}-\frac{a d}{\sqrt{(a+b)(c+d)(a+c)(b+d)}}\right|\leq \frac{2\tau}{\sqrt{(a+b)(c+d)(a+c)(b+d)}}$$
Applying the same argument replacing $ad$ with $bc$ and substituting into (\ref{eq:baz}) gives
$$|\hat\rho -\rho| \leq \frac{4\tau}{\sqrt{(a+b)(c+d)(a+c)(b+d)}}.$$
By assumption $c+d\leq a+b$ hence $a+b\geq 1/2$ and $(a+b)(c+d)\geq \tau/2$, and since $\E[S]=b+d\in [v,1-v]$, we have $(a+c)(b+d) \geq v/2$. Combining with the above gives $|\hat\rho -\rho| \leq 8\sqrt{\tau/v}=\eps.$

Case 3: $c+d\geq 1/2$. Replacing $R$ by $1-R$ negates $\rho$ and also negates $\hat\rho$. This transformation swaps $a$ with $c$ and $b$ with $d$ but preserves $|\rho-\hat\rho|$. Hence, we can use the prior two cases which cover $c+d \leq 1/2$.
\end{proof}

\end{document}